\documentclass[11pt]{article}

\usepackage{amssymb,amsmath,amsthm,amsbsy,amsfonts}
\usepackage{mathtools}
\usepackage[letterpaper,margin=1in]{geometry}
\usepackage[T1]{fontenc}
\usepackage{libertinus}
\usepackage{microtype}
\usepackage{bm}
\usepackage{graphicx}
\usepackage{float}        
\usepackage{booktabs}
\usepackage[font={small},labelfont={bf,up}]{caption}
\usepackage{algorithm}
\usepackage{algpseudocode}
\usepackage{thmtools,thm-restate}
\makeatletter
\newcommand{\algspacing}{\setlength{\@tempdima}{0pt}\renewcommand{\baselinestretch}{1.25}\selectfont}
\makeatother
\algrenewcommand\algorithmicrequire{\textbf{Input:}}
\algnewcommand{\Phase}[1]{\Statex\textbf{#1}}
\usepackage{enumitem}
\usepackage{titletoc}
\usepackage{xcolor}
\usepackage[colorlinks,linkcolor=blue,citecolor=blue,urlcolor=blue]{hyperref}
\usepackage[nameinlink]{cleveref}
\usepackage[firstinits=true,backend=bibtex,style=numeric-comp,citestyle=numeric-comp,sorting=nyt,url=false,arxiv=false,isbn=false,doi=false]{biblatex}
\graphicspath{{figs/}}

\renewcommand{\vec}{\operatorname{vec}}

\newcommand{\R}{\mathbb{R}}

\DeclareMathOperator{\diag}{diag}
\DeclareMathOperator{\Tr}{Tr}
\DeclareMathOperator{\EMA}{EMA}
\DeclareMathOperator*{\argmin}{arg\,min}

\newcommand{\ip}[1]{\langle #1 \rangle}
\newcommand{\DW}{\Delta W}
\DeclareMathOperator{\msign}{msign}          

\newcommand{\norm}[1]{\| #1\|}
\newcommand{\methodname}{PoLoRA\xspace}
\usepackage{xspace}
\theoremstyle{plain}
\newtheorem{lemma}{Lemma}
\newtheorem{proposition}{Proposition}

\theoremstyle{definition}

\theoremstyle{remark}
\newtheorem{remark}{Remark}
\theoremstyle{definition}

\title{\methodname: A Preconditioned Orthogonalized LoRA Optimizer}
\author{
Nikhil Ghosh\footnote{Corresponding author: nghosh@flatironinstitute.org} \hspace{1cm}
Tetiana Parshakova \hspace{1cm}
Robert M. Gower\\[0.5em]
Center for Computational Mathematics, Flatiron Institute\\
New York, NY
}
\date{}

\begin{document}
\maketitle

\begin{abstract}
Low-rank adaptation (LoRA) makes finetuning large language models cheaper by adding to each weight matrix a trainable low-rank update parameterized as the product of two matrices. These matrices are usually trained with Adam, which treats them as a single flat vector of parameters and ignores both the matrix and product structure of LoRA. Applying a matrix-aware optimizer such as Muon to each factor does not consistently improve over Adam, and neither do the product-aware Muon variants proposed in concurrent works. To realize consistent gains, we introduce PoLoRA, a Preconditioned Orthogonalized LoRA optimizer built from three ingredients: a product-aware spectral update direction, curvature preconditioning derived from controlling the per-sample loss change, and a magnitude rule that controls the sizes of both the factor and merged updates. We evaluate PoLoRA on instruction-tuning datasets for code and math across models from 1B to 8B parameters, and find that it reaches the final held-out loss achieved by tuned Adam in 1.2--1.7 times fewer steps, while adding at most 3\% per-step overhead. Compared to Adam, PoLoRA is also less sensitive to the learning rate, and its optimal learning rate is stable across ranks.
\end{abstract}

\section{Introduction}
\label{sec:intro}

Deployment of large-scale pretrained models increasingly depends on adapting a
general model to many specialized tasks~\cite{foundation_models}. Finetuning a separate full copy for each task is expensive, because training must hold gradients and optimizer state for the whole model, and deploying each task means storing and serving multiple full-sized models. Low-rank adaptation (LoRA~\cite{lora}) avoids both issues by freezing the base weight $W_0 \in \R^{d_{\mathrm{out}} \times d_{\mathrm{in}}}$ and adding a small low-rank update $\Delta W$,
\begin{equation}
W = W_0 + \DW, \qquad \DW = \frac{\alpha}{r}\,BA, \qquad
B \in \mathbb{R}^{d_{\mathrm{out}} \times r}, \quad
A \in \mathbb{R}^{r \times d_{\mathrm{in}}},
\label{eq:lora}
\end{equation}
to each layer, with $\alpha$ a scaling parameter and $r$ the rank. With only the adapter $(B, A)$ trained, finetuning can run on nearly the same hardware layouts as inference, and the shared base lets one server serve many adapters at once~\cite{punica,bitdelta,schulman2025lora}.

The standard optimization recipe for LoRA applies Adam~\cite{adam} to the two factors $A$ and $B$, treating all the parameters as a single flat vector and ignoring any matrix structure. In language-model pretraining, matrix-aware optimizers such as Muon~\cite{muon}, Scion~\cite{pmlr-v267-pethick25a}, and Shampoo~\cite{shampoo} treat each weight as a matrix and outperform Adam~\cite{liu2025muon,shampoo,benchmark_practices,optimizer_benchmarking,qiu2026hyperparameter}. Muon in particular and variants such as NorMuon~\cite{Normuon} have set records for training small language models~\cite{nanogpt,nanochat} and have been adopted in recent large-scale training runs~\cite{Deepseekv4,GLM}.

It is natural to ask whether similar gains can carry over to LoRA. We found, however, that applying Muon to each factor does not improve over Adam. Part of the problem is structural: the model sees the \emph{product} $BA$, while the optimizer acts on the \emph{factors} $A$ and $B$ individually. Several concurrent works~\cite{imuon,loramuon,tilde_csd} correct this with variants of a product-aware Muon step we refer to as \emph{Product Muon} (\Cref{sec:product-muon}).
In our experiments, the official iMuon implementation~\cite{imuon}, one of these variants, is also unable to consistently improve over Adam (\Cref{fig:hero}).

\begin{figure}[H]
  \centering
  \includegraphics[width=\linewidth]{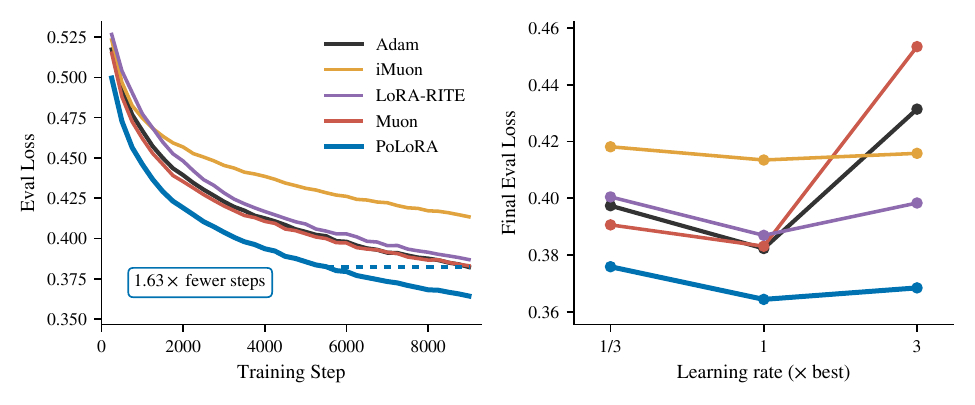}
  \caption{\textbf{\methodname{} outperforms baselines during training.} Llama-3.2-1B finetuned on the math dataset at rank $r = 256$. \textbf{Left:} evaluation loss across training steps; \methodname{} outperforms Adam, iMuon~\cite{imuon}, Muon, and LoRA-RITE~\cite{lorarite} and reaches Adam's final loss in $1.63\times$ fewer steps. \textbf{Right:} final evaluation loss versus learning rate, in multiples of each optimizer's best learning rate. Every optimum is bracketed, and \methodname{} attains the lowest loss.}
  \label{fig:hero}
\end{figure}

To obtain a consistent speedup over Adam, we introduce \methodname{}, a steepest descent method whose constraint is motivated by controlling the change in the per-sample loss, and whose solution is an orthogonalized update preconditioned by a curvature estimate. We then develop a product-aware version of this update for LoRA, imposing the constraint on the merged update rather than on each factor separately. A separate magnitude rule sets the sizes of the factor updates, which the Product Muon step leaves uncontrolled, while retaining its bound on the spectral norm of the merged update.

Across base models, datasets, and adapter ranks, \methodname{} consistently improves training efficiency over Adam (\Cref{tab:breadth,tab:rank}), and it outperforms matrix-aware LoRA baselines on Llama-3.2-1B finetuned on math (\Cref{fig:hero}). It reaches the final held-out loss of tuned Adam in $1.2$--$1.7\times$ fewer steps, while adding at most 3\% per-step overhead. Removing the curvature preconditioning and the magnitude rule in turn reduces \methodname{} to Product Muon, and each component accounts for about half of the speedup gap. Finally, we observe that \methodname{} is easier to tune than Adam: its optimal learning rate is stable across ranks, and performance is less sensitive to the learning rate.

Our open-source implementation is available at:
\begin{quote}\centering
\url{https://github.com/nikhilgsh/polora}
\end{quote}

The rest of the paper proceeds as follows. In \Cref{sec:method} we motivate and derive \methodname{}.
In \Cref{sec:exp} we evaluate \methodname{}, comparing it in terms of step and wall-clock speedups over Adam and other baselines.
In \Cref{sec:exp-ablation} we separate the roles of curvature and magnitude control.
Finally, in \Cref{sec:related} we place \methodname{} in the context of prior and concurrent work on hyperparameter rescaling for LoRA, matrix-aware updates, product-invariant LoRA optimizers, and magnitude control for low-rank products.


\section{\texorpdfstring{\methodname{}}{PoLoRA}: Preconditioned Orthogonalized LoRA}

\label{sec:method}
To derive the \methodname{} update, we first recall in \Cref{sec:muon} how Muon arises from constrained steepest descent~\cite{crawshaw2026exploration}. Its update is computed by a linear minimization oracle (LMO)~\cite{pmlr-v267-pethick25a}, which minimizes the linearized loss subject to a spectral-norm bound. In \Cref{sec:product-muon} we make this LMO product-aware, which yields Product Muon. In \Cref{sec:sample-loss-lmo} we obtain the \methodname{} update direction by replacing the spectral-norm constraint with a bound on the change in the per-sample loss and making the resulting LMO product-aware for LoRA. The LMO provides the direction but not the size of the update, which we set by a separate rescaling that controls both the per-factor and merged-weight updates. All proofs are in \Cref{app:proofs}.

Throughout, whenever a vector or matrix is divided by its norm, we set $X/\norm{X}:=0$ at $X=0$.

\subsection{Muon}\label{sec:muon}
Muon is based on a spectral LMO~\cite{bernstein2025modular,pmlr-v38-carlson15,carlson2015preconditioned}.
For a weight $W \in \mathbb{R}^{d_{\mathrm{out}} \times d_{\mathrm{in}}}$ with gradient $G = \nabla_W \mathcal{L}(W)$, this LMO minimizes the local linearization of the loss subject to a ball constraint,
\begin{equation}
 \begin{array}{ll}
  \underset{\Delta W}{\mathrm{minimize}} &
  \ip{G, \Delta W} \\[2pt]
  \mbox{subject to} &
  \displaystyle
  \norm{\Delta W}_2
  \le \eta,
  \end{array}
  \label{eq:lmo}
\end{equation}
where $\|\cdot\|_2$ is the spectral norm 
\[
\|W\|_2 := \max_{\|x\|_2 =1} \|W x \|_2.
\]
Bounding the spectral norm is appealing because it controls the largest amount by which the update can change the layer's output over all unit-norm inputs~\cite{yang2023spectral}.
A solution to~\eqref{eq:lmo} is given by the matrix sign of the gradient~\cite{bernstein2025modular}
\[
    \Delta W = -\eta \msign(G),
\]
where $\msign(G)=UV^\top$ and $G=U\Sigma V^\top$ is the reduced SVD (see~\Cref{lem:muon-lmo}). 

In practice, Muon replaces the gradient with its momentum, an exponential
moving average (EMA) of past gradients, and uses a polynomial approximation
to the polar factor~\cite{polarexpress,gram_ns} (see \Cref{app:gramns}). Next we show how to adapt this LMO to be aware of the product structure in LoRA~\eqref{eq:lora}.

\subsection{Product Muon}\label{sec:product-muon}
Throughout, we set the LoRA scaling $\alpha=r$~\cite{loralearnsless, raschka}, so the merged weight is $W = W_0 + BA$. Consider a LoRA step that updates the two factors $A$ and $B$ as follows
\[
A \leftarrow A + \Delta A,
\qquad
B \leftarrow B + \Delta B.
\]
Applying Muon naively would impose a spectral bound separately on $\Delta A$ and $\Delta B$, but what matters to the layer is the effective update of the product $BA$, which is not controlled solely by $\Delta A$ and $\Delta B$. Indeed, the update to the merged weight is, to first order,
\begin{equation} \label{eq:DeltaW-lin}
\Delta W = (B + \Delta B)(A + \Delta A) - BA \approx B \Delta A + \Delta B A,
\end{equation}
from which we see that the spectral norm bound should be applied to this linearized product update. Substituting~\eqref{eq:DeltaW-lin} into the spectral LMO~\eqref{eq:lmo} gives the spectral constraint
\[
\norm{B \Delta A + \Delta B A}_2
\le \eta.
\]
The constraint couples the two unknowns $\Delta A$ and $\Delta B$, but we can decouple them with the upper bound using norm subadditivity:
\[
\norm{B \Delta A + \Delta B A}_2
\le
\norm{B \Delta A}_2 + \norm{\Delta B A}_2.
\]
Splitting the budget evenly between the two terms preserves the merged bound and gives the \emph{Product Muon} LMO
\begin{equation}
  \begin{aligned}
  \Delta A &\in \argmin_{\Delta A}~ \ip{G_A, \Delta A}
  \quad \textnormal{subject to}\quad
  \norm{B \Delta A}_2\le\eta/2, \\
  \Delta B &\in \argmin_{\Delta B}~ \ip{G_B, \Delta B}
  \quad \textnormal{subject to}\quad
  \norm{\Delta B A}_2\le\eta/2,
  \end{aligned}
  \label{eq:decoupled-lora-lmo}
\end{equation}
where the factor gradients are
\begin{equation}\label{e-grad-wrt-factors}
    G_A := \nabla_A \mathcal{L}(W_0+BA) = B^\top G,
    \qquad
    G_B := \nabla_B \mathcal{L}(W_0+BA) = G A^\top.
\end{equation}
By \Cref{lem:whitened-spectral-lmo}, a solution to this LMO is
\begin{equation}\label{eq:product-muon-step}
\begin{aligned}
    \Delta A &= -\tfrac{\eta}{2}\,(B^\top B)^{-1/2}\,\msign\!\big((B^\top B)^{-1/2} G_A\big), \\
    \Delta B &= -\tfrac{\eta}{2}\,\msign\!\big(G_B (A A^\top)^{-1/2}\big)\,(A A^\top)^{-1/2}.
\end{aligned}
\end{equation}
As in Muon, these factor updates are computed using momentum estimates of the gradients.
\subsection{Per-Sample Loss Control}\label{sec:sample-loss-lmo}
We now derive the base LMO that will be the starting point of \methodname{}.
Let $\{x_1, \ldots, x_n\}$ be the input activations to the layer that $W$ encodes, and let $G_i = \nabla \ell_{x_i}(W)$ be the gradient of the loss on activation $x_i$. Abusing notation slightly, we write $\ell_{x_i}$ for the loss on a single activation rather than a full input.

The justification for using the spectral norm in Muon is that it is conservative, since the constraint $\norm{\Delta W}_2 \le \eta$ is equivalent to bounding the change in the layer's output,
\[\norm{\Delta W x}_2 \leq \eta \quad \text{for all } \norm{x}_2 \leq 1.\]
What really matters in the end, however, is the final loss. Thus we adapt this conservative approach to instead ensure that no single activation results in a large change in the \emph{loss}, that is, we want
\begin{equation}\label{eq:loss-diff}
     |\ell_{x_i}(W +\Delta W) - \ell_{x_i} (W)|\leq \tau, \qquad  i=1,\ldots, n.
\end{equation} 
Since the loss is nonlinear and nonconvex, we replace it with a local linearization,
\[  
\ell_{x_i}(W +\Delta W) -\ell_{x_i} (W)  \approx \ip{G_i, \Delta W}. 
\]
Imposing the linearized per-sample bounds~\eqref{eq:loss-diff} in place of the spectral-norm constraint in~\eqref{eq:lmo} gives
\begin{equation}
 \begin{array}{ll}
  \underset{\Delta W}{\mathrm{minimize}} &
  \ip{G, \Delta W} \\
  \mbox{subject to} &
  \displaystyle  
|\ip{G_i, \Delta W}|
  \le \tau, \qquad i=1,\ldots, n.
  \end{array}
  \label{eq:sample-lmo}
\end{equation}
This per-sample LMO is well-posed (\Cref{prop:wellposed}), but its constraints are not directly usable, since we lack efficient access to the per-sample gradients $G_i$. We therefore define a centrally symmetric \emph{outergradient} set $\mathcal{G}$ such that $G_i \in \mathcal{G}$, and then use the inequality
\begin{equation} \label{eq:outgrad1}
    \max_{i=1,\ldots, n}  |\ip{G_i, \Delta W}|
    \leq   \max_{\widetilde{G} \in \mathcal{G}} \; \ip{\widetilde{G}, \Delta W},
\end{equation} 
to impose  that the right-hand side of \eqref{eq:outgrad1} 
is less than $\tau$. 
To build $\mathcal{G}$ we collect some standard properties of the gradient matrices in the following lemma.
\begin{restatable}[Properties of gradients]{lemma}{gradprop} 
  \label{lem:gradprop}
    Let $g_i=\vec(G_i) \in \R^{d_{\mathrm{out}} d_{\mathrm{in}}}$ for $i=1, \ldots, n$. Let $\Sigma := \frac{1}{n} \sum_{i=1}^n g_i g_i^\top$ and let $\Sigma^\dagger$ denote its Moore-Penrose pseudoinverse. The gradients satisfy the following properties:
    \begin{enumerate}
        \item {\bf Leverage bound:} $ \ip{\Sigma^\dagger g_i, g_i} \leq n$, \\[-0.5cm]
        \item {\bf Rank one:} $G_i = b_i x_i^\top$  for some $x_i \in \R^{d_{\mathrm{in}}}$ and $b_i \in \R^{d_{\mathrm{out}}}$.
        \end{enumerate}
\end{restatable}

\noindent The above properties motivate the choice of the following outergradient set
\begin{equation} 
\label{eq:outgrad} 
\mathcal{G}(\Sigma) := \left \{ \widetilde{G} = b x^\top ~\bigg |~  
b\in\mathbb{R}^{d_{\mathrm{out}}},
    \ x\in\mathbb{R}^{d_{\mathrm{in}}}, \
\ip{\Sigma^\dagger \vec(\widetilde{G}), \vec(\widetilde{G})} \leq n \right \}.
\end{equation}
Since $G_i \in \mathcal{G}$ by~\Cref{lem:gradprop}, we can use this $\mathcal{G}$ in our upper bound~\eqref{eq:outgrad1}.
The last issue is that $\Sigma$ is too large to store or use directly, so we replace it with a structured approximation.

\paragraph{Kronecker approximation.}
Following~\cite{shampoo,kfac}, we approximate $\Sigma$ by a Kronecker product of symmetric positive definite matrices $P\in \R^{d_{\mathrm{out}} \times d_{\mathrm{out}}}$ and $Q \in\R^{d_{\mathrm{in}} \times d_{\mathrm{in}}}$,
\begin{equation} \label{eq:pq-factor}
     \Sigma  \approx Q \otimes P.
\end{equation}
We treat the \emph{preconditioners} $P$ and $Q$ as given for now, deferring their update rule to \Cref{app:matnormal}.

Using the factorization above in the outergradient set~\eqref{eq:outgrad} and replacing the per-sample constraints in~\eqref{eq:sample-lmo} with the sufficient condition~\eqref{eq:outgrad1}, we arrive at the following LMO
\begin{equation}
 \begin{array}{ll}
  \underset{\Delta W}{\mathrm{minimize}} &
  \ip{G, \Delta W} \\[2pt]
\mbox{subject to} &
  \displaystyle  
\max_{\widetilde{G} \in \mathcal{G}(Q \otimes P)} \ip{\widetilde{G}, \Delta W} 
  \le \tau.
  \end{array}
  \label{eq:polorafull-hard} 
\end{equation}
By~\Cref{lem:pseudogradient-spectral}, the constraint in~\eqref{eq:polorafull-hard} satisfies the identity
\[
\max_{\widetilde{G} \in \mathcal{G}(Q \otimes P)}\ip{\widetilde{G}, \Delta W}
= \sqrt{n} \| P^{1/2} \Delta W Q^{1/2}\|_2,
\]
so that \eqref{eq:polorafull-hard} becomes
\begin{equation}
\begin{array}{ll}
\underset{\Delta W}{\mathrm{minimize}} &
\ip{G, \Delta W} \\[2pt]
\mathrm{subject~ to} &
\displaystyle  
\|P^{1/2} \Delta W Q^{1/2}\|_2 \leq  \tau,
\end{array}
\label{eq:polorafull}
\end{equation}
where we have absorbed the factor $1 / \sqrt{n}$ into $\tau$. Note that the value of $\tau$ does not matter, since we use the LMO only for the direction and rescale the update later. By \Cref{lem:whitened-spectral-lmo}, the solution to the LMO in~\eqref{eq:polorafull} is
\begin{equation}\label{eq:polorafull-update}
    \Delta W =  - \tau P^{-1/2} \msign(P^{-1/2} G Q^{-1/2}) Q^{-1/2}.
\end{equation}

\paragraph{Direction.}
So far we have described an update for the weight matrix $W \in \R^{d_{\mathrm{out}} \times d_{\mathrm{in}}}$. We now adapt it to LoRA~\eqref{eq:lora}, deriving the \emph{direction} first and then setting the \emph{magnitude} of the factor updates. Replacing $\Delta W$ in \eqref{eq:polorafull} with its linearized product update~\eqref{eq:DeltaW-lin} gives the LMO
\begin{equation}
    \begin{array}{ll}
    \underset{\Delta A,\Delta B}{\mathrm{minimize}} &
    \ip{G, B\Delta A+\Delta B A} \\[2pt]
    \mbox{subject to} &
    \norm{P^{1/2}(B\Delta A+\Delta B A)Q^{1/2}}_2
    \le \tau.
    \end{array}
    \label{eq:joint-direction-lmo}
  \end{equation}
Just as for Product Muon in~\Cref{sec:product-muon}, subadditivity of the spectral norm splits the coupled constraint in~\eqref{eq:joint-direction-lmo} into one bound per factor. With the factor gradients~\eqref{e-grad-wrt-factors}, the product LMO then decouples into two factorwise LMOs
\begin{equation}
\setlength{\arraycolsep}{2pt}
\begin{array}{@{}c@{\qquad}c@{}}
\begin{array}{ll}
\underset{\Delta A}{\mathrm{minimize}} &
\ip{G_A,\Delta A} \\[2pt]
\mbox{subject to} &
\norm{P^{1/2}B\Delta A Q^{1/2}}_2 \le \tau ,
\end{array}
&
\begin{array}{ll}
\underset{\Delta B}{\mathrm{minimize}} &
\ip{G_B,\Delta B} \\[2pt]
\mbox{subject to} &
\norm{P^{1/2}\Delta B A Q^{1/2}}_2 \le \tau .
\end{array}
\end{array}
\label{eq:decoupled-direction-lmo}
\end{equation}
Define the $r\times r$ matrices
\begin{equation}
C_B = B^\top P B,
\qquad
C_A = A Q A^\top .
\label{eq:ours-closure}
\end{equation}
By definition of the spectral norm, and using the definitions in~\eqref{eq:ours-closure}, the left-hand sides of the two constraints in \eqref{eq:decoupled-direction-lmo} can be rewritten as
\begin{equation} \label{eq:collapse-constraint}
\norm{P^{1/2}B\Delta A Q^{1/2}}_2=\norm{C_B^{1/2}\Delta A Q^{1/2}}_2,
\qquad
\norm{P^{1/2}\Delta B A Q^{1/2}}_2=\norm{P^{1/2}\Delta B C_A^{1/2}}_2 .
\end{equation}
Since substituting~\eqref{eq:collapse-constraint} into~\eqref{eq:decoupled-direction-lmo} makes each factorwise LMO an instance of \Cref{lem:whitened-spectral-lmo}, the two minimizers are $\Delta A = -\tau D_A$ and $\Delta B = -\tau D_B$, where
\begin{equation}
  \begin{aligned}
  D_A &=
  C_B^{-1/2}\,
  \msign\!\left(C_B^{-1/2}\,G_A\,Q^{-1/2}\right)
  \,Q^{-1/2},\\
  D_B &=
  P^{-1/2}\,
  \msign\!\left(P^{-1/2}\,G_B\,C_A^{-1/2}\right)
  \,C_A^{-1/2}.
  \end{aligned}
  \label{eq:factor-directions}
\end{equation}
give the factor update directions.

\paragraph{Magnitude.}
Having derived the direction, we now set the magnitude directly rather than through the budget $\tau$, controlling the size of the factor updates in spectral norm. Treating $A$ and $B$ symmetrically, we use a single update size $\rho$ for both factors, and set $\norm{\Delta A}_2=\norm{\Delta B}_2=\rho$. Furthermore, like Product Muon, we ensure that the linearized merged update \eqref{eq:DeltaW-lin} has spectral norm bounded by the learning rate $\eta$. By the triangle inequality and submultiplicativity,
\begin{equation}
\norm{\Delta W}_2
\approx
\norm{B\Delta A+\Delta B A}_2
\le
\norm{B}_2\norm{\Delta A}_2
+
\norm{A}_2\norm{\Delta B}_2,
\label{eq:ours-prodcap}
\end{equation}
and so setting
\begin{equation}
\rho = \frac{\eta}{\norm{A}_2+\norm{B}_2}
\label{eq:ours-rho}
\end{equation}
ensures $\norm{\Delta W}_2\leq \eta$. As a result, using the directions $D_A$ and $D_B$ from \eqref{eq:factor-directions} gives the updates
\begin{equation}
\Delta A = -\rho\,\frac{D_A}{\norm{D_A}_2},
\qquad
\Delta B = -\rho\,\frac{D_B}{\norm{D_B}_2}.
\label{eq:polora-update}
\end{equation}

The update size $\rho$ in \eqref{eq:ours-rho} is not invariant to the rescaling $(A,B)\mapsto(cA,c^{-1}B)$, which changes $\norm{A}_2$ and $\norm{B}_2$ while preserving $BA$. Among such rescalings $\rho$ is largest when $\norm{A}_2=\norm{B}_2$, and in our runs we observe that $\norm{B}_2/\norm{A}_2$ settles near $1$ (\Cref{app:gauge}).
\paragraph{Optimizer step.}
\methodname{}, summarized in \Cref{alg:ours}, makes the update~\eqref{eq:polora-update} practical with the following ingredients:
\begin{itemize}[leftmargin=1.5em,itemsep=1pt,topsep=2pt]
  \item \emph{Averaging.} Momentum with look-ahead replaces each factor gradient with an EMA estimate of its full-batch counterpart (line~\ref{ln:mom}), and a second EMA accumulates the vectors $p,q$ that form the preconditioners (lines~\ref{ln:qmom} and~\ref{ln:pmom}; \Cref{app:matnormal}).  \item \emph{Diagonal curvature.} To reduce overhead, we take $P$ and $Q$ diagonal, as in Adafactor~\cite{adafactor} (line~\ref{ln:pqnorm}), and fit them to second moments of the factor gradients by adapting the KL-Shampoo update~\cite{klshampoo} to the diagonal case (lines~\ref{ln:qmom} and~\ref{ln:pmom}; \Cref{app:matnormal}).
  \item \emph{Normalization.} The Kronecker product in~\eqref{eq:pq-factor} is invariant to the rescaling $(P,Q)\mapsto(aP,a^{-1}Q)$. To make the EMA updates for $P$ and $Q$ invariant to this rescaling as well, we normalize each factor so that its largest entry is one (line~\ref{ln:pqnorm}; \Cref{app:matnormal}).
  \item \emph{Fast numerical subroutines.} We estimate the spectral norms in lines~\ref{ln:rho} and~\ref{ln:spectralnormupdate} by power iteration (\Cref{app:smax}). We use Gram Newton--Schulz iterations with PolarExpress coefficients~\cite{polarexpress} to compute the matrix sign and inverse square roots in lines~\ref{ln:Aspectral} and~\ref{ln:Bspectral} (\Cref{app:gramns}).
\end{itemize}

\begin{algorithm}[t]
\caption{\methodname{} step for one LoRA pair $(A,B)$}
\label{alg:ours}
\begin{algorithmic}[1]
\Statex \textbf{Hyperparameters:} learning rate $\eta$, momentum decay $\beta_1$, curvature decay $\beta_2$, numerical-stability constant $\varepsilon$
\Statex \textbf{Initialize:} $M_A, M_B \gets 0$; \quad $p, q \gets \varepsilon\mathbf{1}$
\Statex
\Phase{Update to Momentum Gradients}
\State $\begin{aligned}[t]
  M_A &\gets \beta_1 M_A+(1-\beta_1)G_A, &\quad M_B &\gets \beta_1 M_B+(1-\beta_1)G_B\\[-4pt]
  \widehat M_A &\gets \beta_1 M_A+(1-\beta_1)G_A, &\quad \widehat M_B &\gets \beta_1 M_B+(1-\beta_1)G_B
\end{aligned}$  \label{ln:mom}\Comment{buffer + look-ahead}
\Statex
\Phase{Computing the Direction}
\State $\begin{aligned}[t]
  &P \gets \diag\!\big(p/\norm{p}_\infty\big), &&\qquad Q \gets \diag\!\big(q/\norm{q}_\infty\big)\\[-4pt]
  &C_B \gets B^\top P\, B, &&\qquad C_A \gets A\,Q\, A^\top
\end{aligned}$  \label{ln:pqnorm}
\State $D_A\gets C_B^{-1/2}\,\msign\big(C_B^{-1/2}\,\widehat M_A\,Q^{-1/2}\big)\,Q^{-1/2}$ \Comment{preconditioned polar step~\eqref{eq:factor-directions}} \label{ln:Aspectral}
\State $D_B\gets P^{-1/2}\,\msign\big(P^{-1/2}\,\widehat M_B\,C_A^{-1/2}\big)\,C_A^{-1/2}$  \label{ln:Bspectral}
\Statex
\Phase{Balancing Magnitude}
\State $\rho\gets\eta/(\norm{A}_2+\norm{B}_2)$ \Comment{spectral-norm control~\eqref{eq:ours-rho}} \label{ln:rho}
\State $A\gets A-\rho\,D_A/\max(\norm{D_A}_2,\varepsilon), \qquad B\gets B-\rho\,D_B/\max(\norm{D_B}_2,\varepsilon)$ \label{ln:spectralnormupdate}
\Statex
\Phase{Update to Preconditioners}
\State $q\gets\beta_2\, q+(1-\beta_2)\diag\big(G_A^\top C_B^{-1} G_A\big)/r$ \Comment{coupled estimator~\eqref{eq:ours-klcoupling}}  \label{ln:qmom}
\State $p\gets\beta_2\, p+(1-\beta_2)\diag\big(G_B\, C_A^{-1} G_B^\top\big)/r$  \label{ln:pmom}
\end{algorithmic}
\end{algorithm}
\section{Experiments}
\label{sec:exp}

We now empirically evaluate \methodname{}, comparing it to several baselines in~\Cref{sec:exp-main}, investigating which design choices of \methodname{} are responsible for its efficacy in~\Cref{sec:exp-ablation}, and testing its behavior across ranks and on different data in~\Cref{sec:exp-taskdep}.

\subsection{Setup}
\label{sec:exp-setup}

\paragraph{Comparison.}
For every model, dataset, and rank, we only vary the optimizer and learning rate. All methods use LoRA on every linear layer except the output layer, with the standard initialization $B=0$ and random $A$ \cite{lora}. For each optimizer, we sweep the learning rate under a constant schedule with no weight decay and select the run with the lowest held-out loss.

The main comparisons finetune OLMo-2-1B~\cite{olmo2}, Llama-3.2-1B~\cite{llama3}, Qwen2.5-1.5B~\cite{qwen25}, and Llama-3-8B on code (OpenCoder \cite{opencoder}) and math (OpenMathInstruct-2 \cite{openmathinstruct2}) datasets, using rank $r=256$ unless stated otherwise. Additional experiments in \Cref{sec:exp-taskdep} vary the rank as well as change the finetuning dataset to a low-resource language. \Cref{app:exp} gives the full experimental configurations.

\paragraph{Baselines.} We compare \methodname{} (\Cref{sec:method}) to the following optimizers:
\begin{itemize}[leftmargin=1.5em,itemsep=1pt,topsep=2pt]
  \item \textbf{Adam} \cite{adam} --- our baseline, applied independently to the two factors $A$ and $B$.
  \item \textbf{Muon} \cite{muon} --- the solution of the spectral LMO~\eqref{eq:lmo}, applied independently to each factor.
  \item \textbf{iMuon} \cite{imuon} --- a product-aware spectral LMO~\eqref{eq:imuon-lmo-ambient} applied to momentum in the product space.
  \item \textbf{LoRA-RITE} \cite{lorarite} --- an adaptive LoRA optimizer that preconditions factor gradients in a transformation-invariant basis.
\end{itemize}
We use the authors' official implementations for iMuon and LoRA-RITE and provide additional background on these methods in \Cref{app:memoryless}. 

\subsection{Computational Efficiency}
\label{sec:exp-main}
We score each optimizer by its \emph{steps-to-Adam}, the number of steps its tuned run needs to match the final loss of tuned Adam. The \emph{step speedup} is then the ratio of Adam's total step count to the steps-to-Adam. Across our experiments, \methodname{} reaches this loss in $1.2$--$1.7\times$ fewer steps (\Cref{tab:breadth}). On Llama-3.2-1B finetuned on math, \methodname{} also outperforms Muon, iMuon, and LoRA-RITE, none of which improves over Adam (\Cref{fig:hero}). Additional learning curves can be found in \Cref{app:curves}.

\paragraph{Wall-clock.}\label{sec:exp-walltime} The optimizer FLOPs of each \methodname{} step are independent of the batch size, while the model forward
and backward passes scale with it (\Cref{app:flops}). Thus the relative per-step time shrinks
as the batch grows. At our batch size (\Cref{app:exp}) it adds at most $3\%$ wall-clock overhead compared to Adam on
all four base models (\Cref{tab:breadth}), and the step speedup carries to wall-clock almost
unchanged.  

\begin{table}[t]
\centering\small
\begin{tabular}{lrrrr}
\toprule
 & \multicolumn{2}{c}{Code} & \multicolumn{2}{c}{Math} \\
\cmidrule(lr){2-3}\cmidrule(lr){4-5}
base model & step & wall-clock & step & wall-clock \\
\midrule
OLMo-2-1B & $1.61\times$ & $1.57\times$ & $1.74\times$ & $1.70\times$ \\
Llama-3.2-1B & $1.55\times$ & $1.51\times$ & $1.63\times$ & $1.58\times$ \\
Qwen2.5-1.5B & $1.30\times$ & $1.27\times$ & $1.64\times$ & $1.60\times$ \\
Llama-3-8B & $1.20\times$ & $1.18\times$ & $1.59\times$ & $1.56\times$ \\
\bottomrule
\end{tabular}
\caption{\textbf{\methodname{} speedup over Adam across models and datasets.} Four base models at rank $r=256$, finetuned on code and math datasets. The wall-clock column discounts the step speedup (\Cref{sec:exp-main}) by the per-step overhead (\Cref{app:exp}). \methodname{} improves over Adam in every setting, by a wider margin on math than on code.}
\label{tab:breadth}
\end{table}

\subsection{Component Ablation}
\label{sec:exp-ablation}

The essential ingredients in \methodname{} can be decomposed into three components:
\begin{itemize}[leftmargin=1.7em,itemsep=1pt,topsep=2pt]
  \item \textbf{Product-awareness}. The direction LMO~\eqref{eq:decoupled-lora-lmo} uses the structure of the merged product $BA$ as opposed to treating $A$ and $B$ as independent matrices.
  
  \item \textbf{Curvature}. Preconditioning the gradient by an adaptive curvature estimate makes the update account for how different directions change the loss~\eqref{eq:pq-factor}.
  \item \textbf{Magnitude}. The sizes of the factor steps and the merged update are both controlled~\eqref{eq:ours-rho}.
\end{itemize}
Removing the curvature preconditioning ($P$ and $Q$ set to identity) and the magnitude control ($\tau = \eta/2$ in \eqref{eq:decoupled-direction-lmo}) reduces \methodname{} to Product Muon \eqref{eq:product-muon-step}. On Llama-3.2-1B finetuned on math (\Cref{fig:ablation}), Product Muon has no speedup over Adam, and the curvature and magnitude components each account for about half the speedup gap between Product Muon and \methodname{}.

\begin{figure}[t]
  \centering
  \includegraphics[width=\linewidth]{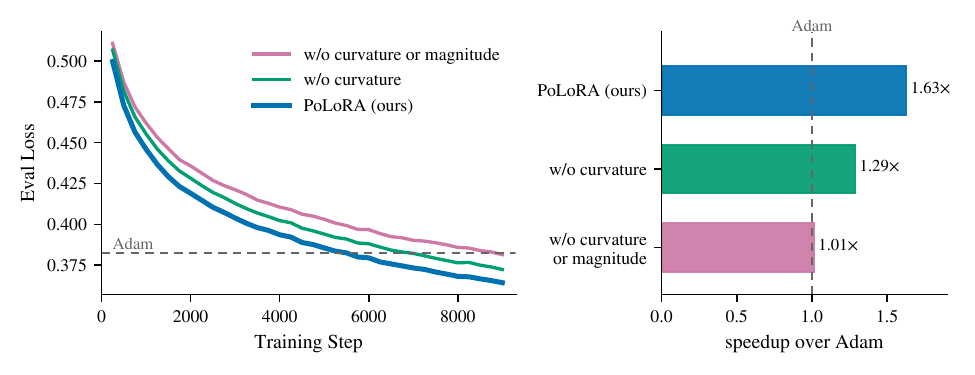}
  \caption{\textbf{Curvature and magnitude both improve \methodname{}.} Component ablation of Llama-3.2-1B finetuned on math at rank $r = 256$. The left panel shows loss curves as the curvature and then the magnitude control are removed from \methodname{} in turn, down to Product Muon. The right panel shows the speedup over Adam for each variant.}
  \label{fig:ablation}
\end{figure}

\subsection{Rank and Distribution Shift}
\label{sec:exp-taskdep}
To probe the effect of the LoRA rank and of the finetuning dataset on the performance of \methodname{}, we vary the rank while finetuning on math and vary the dataset from code to a low-resource language underrepresented in the pretraining mix.

\paragraph{Rank.} We sweep the rank $r$ on Llama-3.2-1B finetuned on math. As the rank grows, the step speedup of \methodname{} over Adam grows too (\Cref{tab:rank}; learning curves in \Cref{fig:curves}(c)). \methodname{} also exhibits learning-rate transfer: its optimal learning rate is stable across ranks, whereas Adam's shifts (\Cref{fig:lr_transfer}). We also observe reduced sensitivity to the learning rate compared to Adam.

\begin{table}[t]
\centering\small
\begin{tabular}{lrr}
\toprule
rank $r$ & step speedup & wall-clock speedup \\
\midrule
$32$  & $1.48\times$ & $1.44\times$ \\
$64$  & $1.54\times$ & $1.50\times$ \\
$128$ & $1.52\times$ & $1.49\times$ \\
$256$ & $1.63\times$ & $1.58\times$ \\
\bottomrule
\end{tabular}
\caption{\textbf{\methodname{} speedup over Adam across LoRA ranks.} Llama-3.2-1B finetuned on the math dataset, with only the rank varied across rows. The wall-clock column discounts the step speedup (\Cref{sec:exp-main}) by the per-step overhead (\Cref{app:exp}).}
\label{tab:rank}
\end{table}

\begin{figure}[t]
  \centering
  \includegraphics[width=\linewidth]{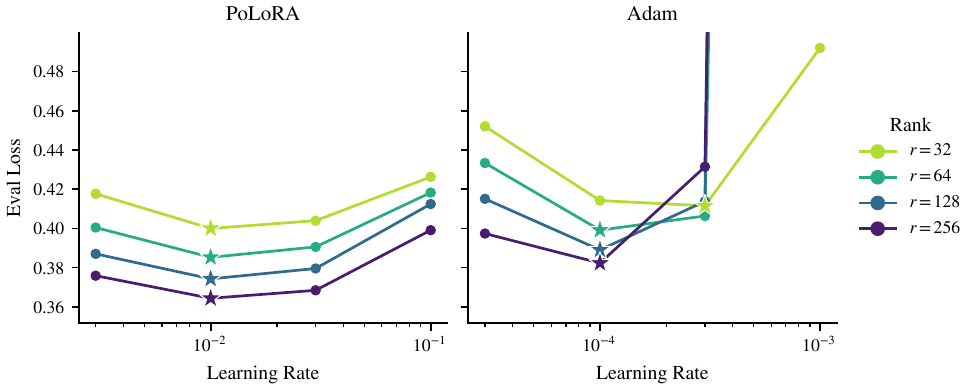}
  \caption{\textbf{\methodname{} exhibits learning-rate transfer.} Final loss versus learning rate of Llama-3.2-1B finetuned on the math dataset at each rank $r\in\{32,64,128,256\}$.
  The optimal learning rate for \methodname{} stays at $10^{-2}$ across ranks, while Adam's drifts.}
  \label{fig:lr_transfer}
\end{figure}

\paragraph{Distribution shift.} Coding data is well represented in the pretraining data of the Qwen2.5-1.5B model~\cite{qwen25}, while data for a low-resource language like Bengali is much less represented. From \Cref{fig:ood} we observe that the speedup over Adam is larger on the Bengali dataset (Aya \cite{aya}) than on code. This is consistent with the advantage of \methodname{} growing as finetuning must move the model further from its pretraining data.

\begin{figure}[t]
  \centering
  \includegraphics[width=\linewidth]{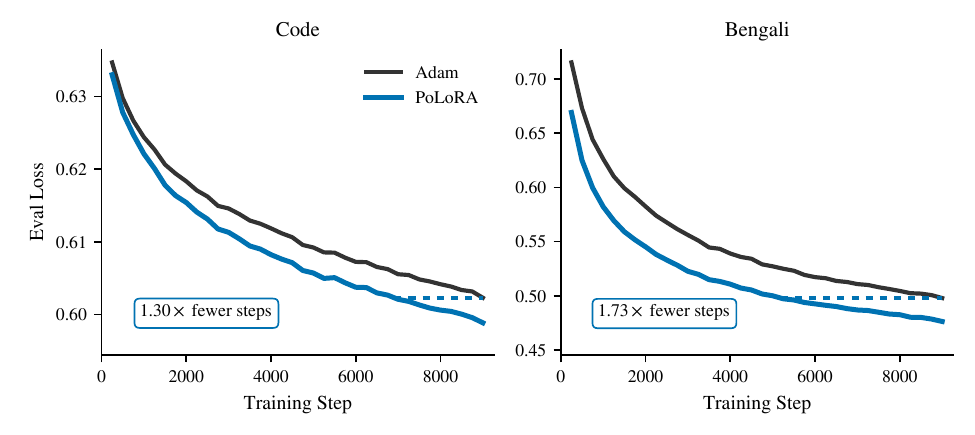}
  \caption{\textbf{The speedup over Adam is larger on the low-resource language.}
  Evaluation loss over training steps on Qwen2.5-1.5B at rank $r = 256$,
  finetuned on the code dataset (left)
  and the Bengali dataset (right).}
  \label{fig:ood}
\end{figure}

\section{Related Work}
\label{sec:related}

Here we review related work and contrast \methodname{} with the existing literature.

\paragraph{Hyperparameter rescaling for LoRA.}
LoRA~\cite{lora} enables parameter-efficient finetuning by freezing the pretrained
weights and training only a low-rank update~\eqref{eq:lora} through factors $A$ and
$B$, typically optimized with Adam~\cite{adam}. Several recent methods refine this
approach by rescaling hyperparameters such as the learning rates, adapter scale, or
initialization variance to account for the low-rank structure. Motivated by large-width
feature-learning arguments, LoRA+~\cite{loraplus} assigns $B$ a larger learning rate
than $A$, and follow-up work~\cite{hayou2024initialization, li2025beyond} extends this
to other initializations. Rank-stabilized LoRA~\cite{rslora} replaces the $\alpha/r$
adapter scale with $\alpha/\sqrt{r}$ to avoid the rank-dependent shrinkage that makes
high-rank adapters hard to train, and $\mu$A~\cite{mua} studies how these
hyperparameters should scale with rank to keep the optimal learning rate stable. In
contrast, \methodname{} keeps these scalar hyperparameters fixed and instead changes the update itself, computing the factor update directions in a curvature-aware metric on the merged product $BA$ rather than rescaling a coordinatewise Adam step.

\paragraph{Matrix-aware updates.}
Recently, optimizers that operate at the matrix level, rather than treating parameters as
a flattened vector, have attracted attention for their gains over Adam~\cite{liu2025muon,
qiu2026hyperparameter, benchmark_practices}. Muon~\cite{muon} solves the spectral-norm LMO
for a weight matrix, producing an update equal to the matrix sign of the momentum
gradient~\cite{polarexpress} (\Cref{lem:muon-lmo}). Shampoo~\cite{shampoo},
SOAP~\cite{soap}, and KL-Shampoo~\cite{klshampoo} estimate the gradient curvature as
Kronecker-factored second moments. Mousse~\cite{mousse} uses such a second moment estimate to
whiten the gradient before a matrix-sign step. \methodname{} likewise applies a matrix-sign step to a preconditioned gradient, with a metric derived from controlling the per-sample loss change across the batch (\Cref{sec:sample-loss-lmo}). We carry this update from the full weight matrix to the LoRA product and add magnitude control, which removes the need for learning-rate grafting~\cite{shi2023distributed, mousse}. We adapt the KL-Shampoo estimator to LoRA, fitting the preconditioners from the factor gradients rather than the weight gradient (\Cref{app:matnormal}).

\paragraph{Product-invariant LoRA optimizers.}
The learned update $BA$ is a product of two matrices, so it is determined only up to the
change of basis $(A,B)\mapsto(NA,BN^{-1})$ for invertible $N$. Since Adam is not invariant
to this transformation, a line of work has proposed LoRA optimizers that are. LoRA-RITE
removes the basis dependence with an adaptive preconditioner~\cite{lorarite}, while Riemannion
optimizes $BA$ directly on the fixed-rank matrix manifold~\cite{riemannion}. iMuon, LoRA-Muon, and Compositional Muon derive invariant updates from
spectral-norm LMOs~\cite{imuon,loramuon,tilde_csd}. As discussed in \Cref{app:memoryless}, the momentum-free LMO updates of LoRA-RITE, LoRA-Muon,
and Compositional Muon's half-split rule all coincide with \eqref{eq:decoupled-lora-lmo}, as does iMuon's intrinsic LMO. \methodname{} does not satisfy this invariance, but it outperforms iMuon in our experiments (\Cref{sec:exp-main}), suggesting that invariance is not the limiting factor in this setting.

\paragraph{Magnitude control for low-rank products.}
Product Muon~\eqref{eq:decoupled-lora-lmo} bounds $\norm{B\Delta A}_2$ and $\norm{\Delta B A}_2$ but not the factor steps $\norm{\Delta A}_2$ and $\norm{\Delta B}_2$. As a result, when $A$ or $B$ has small singular values, the updates $\Delta A$ and $\Delta B$ can be large enough to destabilize training. Spectron~\cite{spectron} and QuacK~\cite{quack} explicitly control the factor steps by scaling them with the factor norms. Spectron rescales both steps to spectral norm $\eta/(\norm{A}_2+\norm{B}_2+1)$, and QuacK scales each factor's learning rate by the inverse norm of the other. Similar to Spectron, \methodname{} rescales both factor steps to a common spectral norm $\rho=\eta/(\norm{A}_2+\norm{B}_2)$. With the $+1$ in its denominator, Spectron's bound also covers the quadratic term $\Delta B \Delta A$, which \methodname{} drops (\Cref{rem:quadratic} verifies that restoring it leaves the bound essentially unchanged).


\section{Conclusion}
\label{sec:conclusion}

\methodname{} takes a different route to optimizing LoRA adapters: it asks what a step does to the merged product $BA$ and to the per-sample losses, rather than treating the two LoRA factors as unrelated parameter matrices. This view leads to a practical optimizer that combines a product-aware spectral direction, lightweight curvature preconditioning, and explicit control of factor-step magnitudes. In our experiments this makes matrix-aware LoRA optimization consistently faster than Adam, and the optimal learning rate transfers across ranks.

The method is still built on approximations. For efficiency the preconditioners are diagonal, and the metric they induce is a Kronecker-factored approximation of the per-sample second moment. The derivation also relies on a linearized loss change and a linearized product update, so the optimizer does not control higher-order effects or the dropped $\Delta B\Delta A$ term. Finally, our evaluation focuses on LoRA finetuning for language models and on held-out loss as the main measure of progress; broader downstream evaluations, longer training horizons, other adapter families, and different hardware or batch-size regimes remain open.

These limitations point to several future directions. One is to design better outergradient approximations that remain cheap enough for adapter training. Another is to incorporate second-order information~\cite{du2026newtonmuonoptimizer} while keeping the update tractable. More broadly, the per-sample loss-control perspective~\eqref{eq:sample-lmo} and~\eqref{eq:polorafull-hard} may be useful beyond LoRA, for pretraining or for other compositions of maps.

\printbibliography

\clearpage
\appendix
\crefname{appendix}{Appendix}{Appendices}
\Crefname{appendix}{Appendix}{Appendices}
\crefalias{section}{appendix}
\crefalias{subsection}{appendix}
\startcontents[app]
\section*{Appendix Contents}
\printcontents[app]{}{1}{\setcounter{tocdepth}{2}}
\clearpage
\section{Experimental Details}
\label{app:exp}
This section gives the full configuration behind \Cref{sec:exp-setup}. Within each setting (a model, a dataset, and a rank), the optimizer and its learning rate are the only things we vary; everything in \Cref{tab:config} is held fixed. For each optimizer we sweep the learning rate and report the run with the lowest held-out loss.

\paragraph{Models.} Four pretrained base models, each adapted with LoRA: OLMo-2-1B \cite{olmo2}, Llama-3.2-1B \cite{llama3}, Qwen2.5-1.5B \cite{qwen25}, and Llama-3-8B \cite{llama3}.

\paragraph{Datasets.} Three instruction-tuning datasets: code (OpenCoder \cite{opencoder}), math (OpenMathInstruct-2 \cite{openmathinstruct2}), and Bengali (Aya \cite{aya}).

\paragraph{Data preparation.}
\begin{itemize}[leftmargin=1.5em,itemsep=1pt,topsep=2pt]
  \item Each dataset is tokenized once per base model, using that model's own tokenizer.
  \item We hold out $1\%$ of each dataset for evaluation, and use this split for both learning-rate selection and all reported losses.
  \item Training sequences are packed into fixed $2048$-token blocks, and evaluation sequences are padded to $2048$ tokens.
  \item Loss is computed on response tokens only; prompt and padding tokens are masked.
\end{itemize}

\paragraph{Fixed training configuration.} \Cref{tab:config} lists the settings held fixed across all runs. We adapt all linear layers except \texttt{lm\_head} to isolate the efficacy of the optimizer. Prior work, however, suggests that adapting only a subset of modules recovers most of the performance~\cite{lorawithoutregret, plop}.
\begin{table}[ht]
\centering\small
\begin{tabular}{@{}ll@{}}
\toprule
setting & value \\
\midrule
sequence length       & $2048$ \\
batch size             & $16$ \\
training steps         & $9000$ \\
LoRA rank $r$          & $256$ (varied in \Cref{sec:exp-taskdep}) \\
LoRA $\alpha$          & $r$ \\
LoRA dropout           & $0$ \\
LoRA $B$ init          & zeros \\
LoRA-adapted layers    & all linear except the output (\texttt{lm\_head}) \\
precision              & bf16 \\
max gradient norm      & $1.0$ (disabled for LoRA-RITE) \\
LR schedule            & constant \\
\bottomrule
\end{tabular}
\caption{Training configuration, held fixed across optimizers and settings except where noted.}
\label{tab:config}
\end{table}

\paragraph{Optimizer settings.} The learning-rate grid is spaced by factors of three, wide enough to bracket each optimizer's optimal learning rate. All other optimizer hyperparameters are fixed. \Cref{tab:optim-config} lists them for the baseline methods, and \methodname{} uses:
\begin{itemize}[leftmargin=1.5em,itemsep=1pt,topsep=2pt]
  \item momentum decay $\beta_1=0.9$ and curvature decay $\beta_2=0.99$,
  \item $8$ Gram Newton--Schulz iterations (\Cref{app:gramns}),
  \item $8$ power iterations per spectral-norm estimate (\Cref{app:smax}),
  \item relative damping $\delta=10^{-4}$ (\Cref{app:damping}),
  \item numerical-stability constant $\varepsilon=10^{-12}$ (\Cref{app:init}).
\end{itemize}
\begin{table}[ht]
\centering\small
\begin{tabular}{@{}lcccc@{}}
\toprule
 & Adam & Muon & iMuon & LoRA-RITE \\
\midrule
$\beta_1$                 & $0.9$     & $0.9$ & $0.95$ & $0.9$     \\
$\beta_2$                 & $0.999$   & ---   & ---    & $0.999$   \\
$\varepsilon$             & $10^{-8}$ & ---   & ---    & $10^{-6}$ \\
Newton--Schulz iterations & ---       & $8$   & $5$    & ---       \\
\bottomrule
\end{tabular}
\caption{Baseline hyperparameters. $\beta_1$ is the momentum decay, $\beta_2$ the second-moment decay, and $\varepsilon$ the numerical-stability constant.}
\label{tab:optim-config}
\end{table}

\paragraph{Baseline implementations.} We run iMuon and LoRA-RITE through the authors' official implementations, with the hyperparameters in \Cref{tab:optim-config}. LoRA-RITE is invariant to reparameterizations of the LoRA factors \cite[Def.~1]{lorarite}. A global gradient-norm clip would break this invariance, so we run LoRA-RITE without it.

\paragraph{Hardware.} Each run uses a single GPU. Wall-clock numbers (\Cref{sec:exp-walltime}) are measured on an NVIDIA RTX PRO 6000 (Blackwell) with \texttt{torch.compile} enabled. Reproducing all our experiments takes about $850$ GPU-hours.

\section{Proofs}
\label{app:proofs}

We first collect the gradient properties used to motivate the outergradient set~\eqref{eq:outgrad}.
\gradprop*
\begin{proof}[Proof of \Cref{lem:gradprop}]
For the {\bf Leverage bound}, let $g_i=\vec(G_i)$ and
$M=[g_1,\ldots,g_n]$, so $g_i = Me_i$ for every $i =1,\ldots, n$. Then
\[
\ip{\Sigma^\dagger g_i, g_i} = ng_i^\top (MM^\top)^{\dagger} g_i = n e_i^\top M^\top  (MM^\top)^{\dagger} M e_i  \leq n,
\]
since $M^\top  (MM^\top)^{\dagger} M$ is the orthogonal projector onto the range
of $M^\top$.

The {\bf Rank one} property is the chain rule: $W$ encodes a linear layer
$x_i \mapsto W x_i$, so
\[
\nabla \ell_{x_i}(Wx_i) =  \left. \frac{d }{dz} \ell_{x_i}(z)
\right|_{z =Wx_i} x_i^\top
=b_i x_i^\top,
\]
where $b_i = \left.\frac{d }{dz} \ell_{x_i}(z)  \right|_{z =Wx_i}$ is the
output-side gradient.
\end{proof}

We verify that the per-sample LMO~\eqref{eq:sample-lmo} is well-posed.
\begin{proposition}[Well-posedness of the per-sample LMO]
\label{prop:wellposed}
Let $\tau\ge 0$, let $G_1,\dots,G_n$ be the per-sample gradients in~\eqref{eq:sample-lmo}, and let $G=\tfrac1n\sum_{i=1}^n G_i$. Then the per-sample LMO~\eqref{eq:sample-lmo} attains a finite minimum.
\end{proposition}
\begin{proof}
The feasible set is a nonempty polyhedron, since it contains $\Delta W=0$. For every feasible $\Delta W$,
\[
|\ip{G,\Delta W}|\le\tfrac1n\sum_{i=1}^n|\ip{G_i,\Delta W}|\le\tau,
\]
so the objective $\ip{G,\Delta W}$ is bounded below by $-\tau$. A linear function bounded below on a nonempty polyhedron attains its minimum~\cite{bertsimas1997introduction}, which gives the claim. The minimizer need not be unique, since the objective is constant along $\operatorname{span}\{G_1,\ldots,G_n\}^\perp$.
\end{proof}

Next we solve the spectral-norm LMO whose solution is the Muon update direction (\Cref{sec:muon}).
\begin{lemma}[Spectral norm LMO]\label{lem:muon-lmo}
Let $G\in\mathbb{R}^{m\times n}$ and $\eta>0$, and write
$\norm{G}_{\mathrm{nuc}}=\sum_i\sigma_i(G)$ for the nuclear norm, the sum of the
singular values. Then
\[
    \min_{\norm{X}_2\leq \eta}\ip{G,X}
    =
    -\eta\norm{G}_{\mathrm{nuc}},
\]
and $X^\star = -\eta\msign(G)$ is the unique minimizer of least Frobenius norm.
\end{lemma}

\begin{proof}[Proof of \Cref{lem:muon-lmo}]
We lower-bound the objective by norm duality and exhibit a feasible point that
attains it. By duality of the spectral and nuclear norms, every feasible $X$
satisfies
\[
    \ip{G,X}
    \geq
    -\norm{G}_{\mathrm{nuc}}\norm{X}_2
    \geq
    -\eta\norm{G}_{\mathrm{nuc}}.
\]
Let $G=U\Sigma V^\top$ be a reduced singular value decomposition, so
$\Sigma=\diag(\sigma_i(G))$. The choice
\[
    X=-\eta UV^\top=-\eta\msign(G)
\]
is feasible and satisfies
\[
    \ip{G,X}
    =
    -\eta\Tr(\Sigma)
    =
    -\eta\norm{G}_{\mathrm{nuc}},
\]
so it attains the lower bound and is therefore optimal. The minimum-norm
statement follows from Lemma C.2~\cite{parshakova2026muon}.
\end{proof}

Using those properties, we evaluate the outergradient support function under a Kronecker-factored metric, turning the per-sample loss constraint (\Cref{sec:sample-loss-lmo}) into a preconditioned spectral norm.
\begin{lemma}[Outergradient set under a Kronecker factorization]
\label{lem:pseudogradient-spectral}
Let
$P\in\R^{d_{\mathrm{out}}\times d_{\mathrm{out}}}$ and
$Q\in\R^{d_{\mathrm{in}}\times d_{\mathrm{in}}}$ be symmetric positive definite,
and let $\Sigma=Q\otimes P$. Recall the outergradient set~\eqref{eq:outgrad},
\[
\mathcal{G}(\Sigma)=\bigl\{\, bx^\top ~\big|~ b\in\R^{d_{\mathrm{out}}},\ x\in\R^{d_{\mathrm{in}}},\ \ip{\Sigma^\dagger\vec(bx^\top),\vec(bx^\top)}\leq n \,\bigr\}.
\]
Then for every $\Delta W\in\R^{d_{\mathrm{out}}\times d_{\mathrm{in}}}$,
\[
  \max_{\widetilde{G} \in \mathcal{G}(\Sigma)} \ip{ \widetilde{G} , \Delta W}
  =\sqrt{n}\norm{P^{1/2}\Delta W Q^{1/2}}_2 .
\]
\end{lemma}

\begin{proof}[Proof of \Cref{lem:pseudogradient-spectral}]
We evaluate the support function by a change of variables that turns the
constraint into a product of Euclidean balls. Since $P$ and $Q$ are positive
definite, $\Sigma^\dagger=\Sigma^{-1}=Q^{-1}\otimes P^{-1}$. For
$bx^\top\in\mathcal{G}(\Sigma)$, column-major vectorization gives
$\vec(bx^\top)=x\otimes b$, so
\[
\ip{\Sigma^\dagger\vec(bx^\top),\vec(bx^\top)}
=(x\otimes b)^\top(Q^{-1}\otimes P^{-1})(x\otimes b)
=(x^\top Q^{-1}x)(b^\top P^{-1}b).
\]
Writing $b=P^{1/2}u$ and $x=Q^{1/2}v$, the constraint becomes
$\norm{u}_2^2\norm{v}_2^2\leq n$ and the objective becomes
\[
\ip{bx^\top,\Delta W}=b^\top\Delta W x=u^\top P^{1/2}\Delta W Q^{1/2}v .
\]
Maximizing over the balls gives the spectral norm,
\[
\max_{\widetilde{G}\in\mathcal{G}(\Sigma)}\ip{\widetilde{G},\Delta W}
=\max_{\norm{u}_2\norm{v}_2\leq\sqrt{n}}u^\top P^{1/2}\Delta W Q^{1/2}v
=\sqrt{n}\,\norm{P^{1/2}\Delta W Q^{1/2}}_2. \qedhere
\]
\end{proof}

We then solve that preconditioned spectral-norm LMO, which sets the \methodname{} direction.
\begin{lemma}[Preconditioned spectral-norm LMO]
\label{lem:whitened-spectral-lmo}
Let $L$ and $R$ be symmetric positive definite. For a matrix $M$ and $\tau>0$,
\begin{equation}\label{e-gen-spectral-lmo-update}
X^\star = -\tau L^{-1}
\msign\left(L^{-1}MR^{-1}\right) R^{-1}
\end{equation}
is a solution of
\[
\begin{array}{ll}
\underset{\|LXR\|_2\leq\tau}{\mathrm{minimize}}&
\ip{M,X}.
\end{array}
\]
Moreover, $X^\star$ is the unique minimizer of
$\|LXR\|_F$ over the solution set.
\end{lemma}
\begin{proof}[Proof of \Cref{lem:whitened-spectral-lmo}]
We change variables to reduce the preconditioned LMO to the plain spectral-norm
LMO of \Cref{lem:muon-lmo}. Set $Y=LXR$, so $X=L^{-1}YR^{-1}$; the constraint
becomes $\norm{Y}_2\le\tau$ and the objective becomes
\[
\ip{M,X}=\ip{L^{-1}MR^{-1},Y} .
\]
By \Cref{lem:muon-lmo} with budget $\tau$, the unique least Frobenius-norm
minimizer is
\[
    Y=-\tau\,\msign(L^{-1}MR^{-1}).
\]
Substituting $X=L^{-1}YR^{-1}$ gives the stated solution.
\end{proof}

\begin{remark}[Second-order term in the merged update]
\label{rem:quadratic}
The magnitude rule bounds the \emph{linearized} merged update; here we check that
restoring the second-order term dropped in that linearization leaves the bound
essentially unchanged. Writing $s=\norm{A}_2+\norm{B}_2$, the
rule~\eqref{eq:ours-rho} sets $\rho=\eta/s$ to bound the linearized update
$B\Delta A+\Delta B A$ of~\eqref{eq:DeltaW-lin}. The actual update also contains
the quadratic term $\Delta B\Delta A$ dropped there,
\[
  \Delta W = B\Delta A+\Delta B A+\Delta B\Delta A .
\]
With $\norm{\Delta A}_2=\norm{\Delta B}_2=\rho$, the triangle inequality and
submultiplicativity give
\[
  \norm{\Delta W}_2\le\rho s+\rho^2=\eta\Bigl(1+\tfrac{\eta}{s^2}\Bigr),
\]
so the actual update exceeds the target $\eta$ by at most the relative amount
$\eta/s^2$. At the standard initialization $B=0$ with $A$ having independent
entries of variance $\Theta(1/d_{\mathrm{in}})$, we have $s=\norm{A}_2=\Theta(1)$,
so the overshoot is $O(\eta)$ and shrinks as the factors grow.
\end{remark}

\section{Momentum-Free Reductions to Product Muon}
\label{app:memoryless}
In this appendix we show that the LMO updates of LoRA-RITE~\cite{lorarite}, LoRA-Muon~\cite{loramuon}, and Compositional Muon's half-split rule~\cite{tilde_csd} all reduce to Product Muon~\eqref{eq:decoupled-lora-lmo} once their moment accumulation is removed, as does the \emph{intrinsic spectral LMO} that Corollary~4.1 of~\cite{imuon} states for iMuon. The officially implemented iMuon optimizer, however, differs (see~\eqref{eq:imuon-implemented}).

\paragraph{Setup.} As in \Cref{sec:method}, 
let $G=\nabla_W \mathcal{L}(W) \big |_{W=W_0+BA}$ be the gradient with respect to the weight matrix, $G_A=B^\top G$ and $G_B=GA^\top$ be the factor gradients, 
and $S_A=AA^\top$ and
$S_B=B^\top B$ be the $r\times r$ factor Gram matrices.
The two factorwise LMOs in Product Muon~\eqref{eq:decoupled-lora-lmo} with budget $\tau$ are
\begin{equation}
  \begin{aligned}
  \Delta A &\in    \argmin_{\Delta A}~ \ip{G_A, \Delta A} \quad \textnormal{subject to}\quad
\norm{B \Delta A}_2\le\tau \\
\Delta B &\in   \argmin_{\Delta B}~ \ip{G_B,  \Delta B} \quad \textnormal{subject to}\quad  \norm{\Delta B A}_2\le\tau.
  \end{aligned}
  \label{eq:decoupled-lora-lmo-app} 
\end{equation}
We will use the fact that for any matrix $M$ with full column rank and reduced SVD $M=U\Sigma V^\top$,
the matrix sign equals
\begin{equation}
  M(M^\top M)^{-1/2}=
  UV^\top=\msign(M).
  \label{eq:polar-id}
\end{equation}

\paragraph{LoRA-RITE.}
Algorithm~1 of LoRA-RITE~\cite{lorarite} adds three pieces of memory to a base step. These are an EMA first moment with parameter $\beta_1$, an accumulated second moment transported across the changing factor basis, and an ``escaped-mass'' floor added to that second moment.
Dropping all three, we get the authors' momentum-free update (their Eq.~(14)).
In our notation, with $\widehat G_B := G_B S_A^{-1/2}$, it becomes
\[
  \Delta B=-\eta \widehat G_B (\widehat G_B^\top\widehat G_B)^{-1/2} S_A^{-1/2}.
\]
With no accumulation, by \eqref{eq:polar-id} the middle factor is the matrix sign, so
\begin{equation*}
  \Delta B=-\eta \msign(\widehat G_B) S_A^{-1/2}
  =-\eta \msign \big(G_B S_A^{-1/2}\big) S_A^{-1/2},
\end{equation*}
which is the $B$-update solving the Product Muon LMO~\eqref{eq:decoupled-lora-lmo-app} with budget $\tau=\eta$. Analogous arguments hold for the $A$-update.

\paragraph{iMuon, momentum-free.}
iMuon relies on a choice of \emph{ambient matrix} $H\in\R^{d_{\mathrm{out}}\times d_{\mathrm{in}}}$. The intrinsic spectral LMO of Corollary~4.1 maps this ambient matrix to the factor updates as follows
\begin{equation}
  \Delta B(H) =\msign \big(HA^\top S_A^{-1/2}\big)S_A^{-1/2},\qquad
  \Delta A(H) =S_B^{-1/2}\msign \big(S_B^{-1/2}B^\top H\big).
  \label{eq:imuon-lmo-ambient}
\end{equation}
For the base Algorithm~1 in~\cite{imuon}, the ambient matrix is the raw gradient $G$, so~\eqref{eq:imuon-lmo-ambient}
again gives the Product Muon update directions of~\eqref{eq:product-muon-step}. Our experiments benchmark the official iMuon implementation from Appendix~K of~\cite{imuon}, which differs from this momentum-free decoupled update, as we show next.

\paragraph{iMuon, with momentum.} The official implementation (their Appendix~K) does not pass the raw gradient to the intrinsic LMO~\eqref{eq:imuon-lmo-ambient}. It maintains momentum buffers $m_A$ and $m_B$ with parameter $\beta$, forms the look-aheads $\widetilde{G}_A=G_A+\beta m_A$ and $\widetilde{G}_B=G_B+\beta m_B$, and uses as input the combined ambient matrix
\begin{equation*}
  \widehat M=\widetilde{G}_B A+B\widetilde{G}_A .
\end{equation*}
Turning momentum off ($\beta=0$) removes the buffers but not the recombination:
\begin{equation*}
  \widehat M=G_B A+B G_A=G\,A^\top A+BB^\top G\ \neq\ G .
\end{equation*}
Plugging this $\widehat M$ into \eqref{eq:imuon-lmo-ambient} and using $G_A A^\top=B^\top G_B$ gives
\begin{equation}\label{eq:imuon-implemented}
  \widehat M A^\top=G_B S_A+(BB^\top)G_B
  \quad\Longrightarrow\quad
  \Delta B(\widehat M)=\msign\!\big(G_B S_A^{1/2}+(BB^\top)G_B S_A^{-1/2}\big)S_A^{-1/2},
\end{equation}
which differs from the Product Muon update~\eqref{eq:product-muon-step}. The argument for $\Delta A$ is symmetric. Thus Product Muon is the momentum-free limit of the iMuon theory (Corollary~4.1), but not of the implemented optimizer.

\paragraph{LoRA-Muon.}
The momentum-free LMOs of the concurrent LoRA-Muon paper~\cite{loramuon} (its Eqs.~(7) and~(8)) are precisely the decoupled factorwise LMOs~\eqref{eq:decoupled-lora-lmo}. LoRA-Muon further shows that the resulting update is equivalent to the simplified QR-coordinate form of LoRA-RITE.

\paragraph{Compositional Muon.}
Compositional Muon~\cite{tilde_csd} derives product-aware spectral updates for the attention products $W_QW_K^\top$ and $W_OW_V$. For the QK product, its half-split rule imposes the two constraints $\norm{\Delta W_Q W_K^\top}_2\le\epsilon/2$ and $\norm{W_Q\Delta W_K^\top}_2\le\epsilon/2$ separately, and the resulting LMOs are solved by
\[
  \Delta W_Q=-\tfrac{\epsilon}{2}\,\msign\!\big(G_Q C_K^{-1}\big)\,C_K^{-1},
  \qquad C_K=(W_K^\top W_K)^{1/2},
\]
where $G_Q$ is the gradient with respect to $W_Q$, and symmetrically for $W_K$. Setting $W_Q=B$ and $W_K=A^\top$ gives $C_K=(AA^\top)^{1/2}$ and $G_Q=G_B$, and the pair of updates becomes the Product Muon step~\eqref{eq:product-muon-step} with $\eta=\epsilon$.

\section{Estimating the Preconditioners}
\label{app:matnormal}
In \Cref{sec:sample-loss-lmo} we treated the preconditioners $P$ and $Q$ as given when writing the update~\eqref{eq:polorafull-update}. Here we describe how we compute them using online updates from the factor gradients and a normalization that makes these updates invariant to the rescaling $(P,Q)\mapsto(aP,a^{-1}Q)$, arriving at the $p,q$ updates~\eqref{eq:ours-klcoupling} of \Cref{alg:ours}.

\paragraph{Notation.} Consider a given linear layer. Let $t$ denote the current optimizer step. Whenever it is clear from context, we suppress the subscript $t$. For each step $s\le t$,
let $A_s,B_s$ be the LoRA factors and
$G_s\in\R^{d_{\mathrm{out}}\times d_{\mathrm{in}}}$ the gradient of the batch loss with respect to the weight matrix $W_s=W_0+B_sA_s$. Write
\[
  G_{A,s}=B_s^\top G_s\in\R^{r\times d_{\mathrm{in}}},\qquad
  G_{B,s}=G_sA_s^\top\in\R^{d_{\mathrm{out}}\times r}
\]
for the factor
gradients~\eqref{e-grad-wrt-factors}. We use two standard properties of the Kronecker product \cite{magnus2019matrix}: the mixed-product rule $(U \otimes V)(W \otimes Z) = UW \otimes VZ$, and $\vec(X^\top)=\mathcal{K}\vec(X)$, where the commutation matrix $\mathcal{K}$ is a permutation matrix satisfying $\mathcal{K}(U \otimes V)\mathcal{K}^\top = V \otimes U$ for square $U$ and $V$.

\paragraph{EMA averaging.} For any sequence $x_s$, write the EMA of the sequence up to step $t$ as
\[
  \EMA^{\beta}_{s\le t}[x_s]:=(1-\beta)\sum_{s\le t}\beta^{t-s} x_s,
\]
with decay factor $\beta$. Just as we replace the factor gradients $G_A$ and $G_B$ in the objective of \Cref{sec:sample-loss-lmo} with averaged versions $\widehat M_A$ and $\widehat M_B$ (line~\ref{ln:mom} of \Cref{alg:ours}), we replace the sample-gradient second-moment matrix $\Sigma$ (\Cref{lem:gradprop}) with the EMA of the outer products of the vectorized full gradients,
\[
  \overline \Sigma_t:=\EMA^{\beta_2}_{s\le t}[\vec(G_s) \vec(G_s)^\top],
\]
with $\beta_2$ the curvature decay of \Cref{alg:ours}.

\paragraph{Factor-gradient moments.}
As in \Cref{sec:sample-loss-lmo}, to make the update tractable we approximate $\overline \Sigma_t \approx Q_t \otimes P_t$~\eqref{eq:pq-factor} with diagonal matrices $P_t$ and $Q_t$. Backpropagation gives efficient access only to the factor gradients, so we work with their averaged second moments,
\begin{equation} \label{eq:factor-moments}
  \overline \Sigma_{A,t}:=\EMA^{\beta_2}_{s\le t}\bigl[\vec(G_{A,s})\vec(G_{A,s})^\top\bigr],
  \qquad
  \overline \Sigma_{B,t}:=\EMA^{\beta_2}_{s\le t}\bigl[\vec(G_{B,s}^\top)\vec(G_{B,s}^\top)^\top\bigr].
\end{equation}
Suppose the Kronecker approximation is exact, that is $\overline \Sigma_t = Q \otimes P$.

Using that $G_{A,s}=B_s^\top G_s$ and $G_{B,s}^\top = A_s G_s^\top$, it follows by vectorizing that
\[
  \vec(G_{A,s}) = (I \otimes B_s^\top)\vec(G_s),
  \qquad
  \vec(G_{B,s}^\top) = (I \otimes A_s)\,\vec(G_s^\top)
    = (I \otimes A_s)\,\mathcal{K}\vec(G_s).
\]
We now make an approximation by replacing $A_s$ and $B_s$ with their current values $A_t$ and $B_t$ in the above, that is we will use
\[
  \vec(G_{A,s}) \approx (I \otimes B_t^\top)\vec(G_s),
  \qquad
  \vec(G_{B,s}^\top) 
    \approx (I \otimes A_t)\,\mathcal{K}\vec(G_s).
\]
 With this approximation, and 
using the above in~\eqref{eq:factor-moments} gives
\begin{equation}\label{eq:observable-moments}
\begin{aligned}
  \overline \Sigma_{A,t} &= (I \otimes B_t^\top)(Q \otimes P)(I \otimes B_t) = Q\otimes C_{B_t}(P),\\
  \overline \Sigma_{B,t} &= (I \otimes A_t)\,\mathcal{K}(Q \otimes P)\mathcal{K}^\top(I \otimes A_t^\top) = (I \otimes A_t)(P \otimes Q)(I \otimes A_t^\top) = P\otimes C_{A_t}(Q),
\end{aligned}
\end{equation}
where $C_{B_t}(P):=B_t^\top P B_t$ and $C_{A_t}(Q):=A_t Q A_t^\top$.

\paragraph{Fitting the preconditioners.} Fitting the observed moments
to the predictions of the Kronecker model
in~\eqref{eq:observable-moments} requires a measure of discrepancy. Following
KL-Shampoo~\cite{klshampoo}, we use the log-determinant
divergence~\cite{logdetdiv}, which up to terms not involving $S$ is
\[
  D(\Sigma,S)=\Tr(\Sigma S^{-1})+\log\det S.
\]
We fit each preconditioner with the other held fixed:
\begin{equation}\label{eq:factor-fit}
  \hat q(P):=\argmin_{q>0}
    D\bigl(\overline \Sigma_{A,t},\diag(q)\otimes C_{B_t}(P)\bigr),
  \qquad
  \hat p(Q):=\argmin_{p>0}
    D\bigl(\overline \Sigma_{B,t},\diag(p)\otimes C_{A_t}(Q)\bigr).
\end{equation}
A self-consistent pair $(p^*,q^*)$ satisfies $q^*=\hat q(\diag(p^*))$ and $p^*=\hat p(\diag(q^*))$. With $P$ fixed, the first objective in~\eqref{eq:factor-fit} separates over the coordinates of $q$. Minimizing each coordinate yields a closed-form solution for $\hat q(P)$, and likewise for $\hat p(Q)$:
\begin{equation}\label{eq:conditional-minimizers}
\begin{aligned}
  \hat q(P)&=\frac1r\diag\left(\EMA^{\beta_2}_{s\le t}\bigl[G_{A,s}^\top C_{B_t}(P)^{-1}G_{A,s}\bigr]\right),\\
  \hat p(Q)&=\frac1r\diag\left(\EMA^{\beta_2}_{s\le t}\bigl[G_{B,s} C_{A_t}(Q)^{-1}G_{B,s}^\top\bigr]\right).
\end{aligned}
\end{equation}
Each minimizer depends on the other preconditioner through
$C_{B_t}(P)$ or $C_{A_t}(Q)$.

\paragraph{Online update.} Re-evaluating the minimizers
in~\eqref{eq:conditional-minimizers} at each step would require
storing every past factor gradient, since $C_{B_t}(P)$ and $C_{A_t}(Q)$ change
as $P$ and $Q$ are updated. \methodname{} instead evaluates each factor
gradient once, at its own step, and accumulates the results in the vectors
$p_t$ and $q_t$:
\begin{equation}\label{eq:ours-klcoupling}
\begin{alignedat}{2}
  \hat{q}_t(P_t)
    &= \frac1r\diag\bigl(G_{A,t}^\top C_{B_t}(P_t)^{-1}G_{A,t}\bigr),
  &\qquad
  \hat{p}_t(Q_t)
    &= \frac1r\diag\bigl(G_{B,t} C_{A_t}(Q_t)^{-1}G_{B,t}^\top\bigr),
  \\
  q_{t+1}
    &= \beta_2q_t+(1-\beta_2)\hat{q}_t(P_t),
  &\qquad
  p_{t+1}
    &= \beta_2p_t+(1-\beta_2)\hat{p}_t(Q_t),
  \\
  Q_{t+1}
    &= \diag\bigl(q_{t+1}/\norm{q_{t+1}}_\infty\bigr),
  &\qquad
  P_{t+1}
    &= \diag\bigl(p_{t+1}/\norm{p_{t+1}}_\infty\bigr).
\end{alignedat}
\end{equation}
Up to the exponentially decayed initialization $\beta_2^{t+1}q_0$, unrolling the
$q_{t+1}$ update in~\eqref{eq:ours-klcoupling} gives the
minimizer $\hat q(P)$ of~\eqref{eq:conditional-minimizers} with the
step-$s$ matrix $C_{B_s}(P_s)$ in place of a fixed $C_{B_t}(P)$, and likewise for $p_{t+1}$. In \Cref{alg:ours}, the preconditioner updates (lines~\ref{ln:qmom}
and~\ref{ln:pmom}) reuse the matrices $C_{B_t}$ and $C_{A_t}$ already computed for the direction step (line~\ref{ln:pqnorm}), so each step inverts $C_{B_t}$ and $C_{A_t}$ once rather than twice.

\paragraph{Normalization.} The pair $(P, Q)$ enters the
approximation of $\overline \Sigma_t$ only through the product
$Q \otimes P$, and
\[
  (a^{-1}Q)\otimes(aP)=Q\otimes P
  \qquad\text{for every } a>0,
\]
so we want the optimizer to behave identically at $(P,Q)$ and
$(aP,a^{-1}Q)$. The update direction is the same at both, but
\[
  \hat q_t(aP_t)=a^{-1}\hat q_t(P_t),
  \qquad
  \hat p_t(a^{-1}Q_t)=a\,\hat p_t(Q_t),
\]
so we normalize $q_{t+1},p_{t+1}$ in~\eqref{eq:ours-klcoupling} to
make the $p,q$ updates invariant as well.

\clearpage
\section{Implementation Details}
\label{app:primitives}
This section details how \Cref{alg:ours} initializes its state and implements
its numerical subroutines.

\paragraph{Initialization.}\label{app:init}
The EMA vectors start at $p = q = \varepsilon\mathbf{1}$, which the
normalization in line~\ref{ln:pqnorm} maps to the identity metric at the first
step. Since $B=0$ at initialization, $G_A$, $\widehat M_A$, and $C_B$ all
vanish, so $D_A=0$ and $A$ stays fixed (line~\ref{ln:spectralnormupdate}) until
$B$ first becomes nonzero. The $B$ factor updates from the first step, with
$C_A = AQA^\top = AA^\top$, so its first update coincides in direction with the
Product Muon step~\eqref{eq:product-muon-step}.

\paragraph{Relative damping.}\label{app:damping}
Each inverse in \Cref{alg:ours} is damped so that directions of
nearly zero curvature do not blow up. For a positive
semidefinite matrix $C$ ($C_A$ or $C_B$ in \eqref{eq:ours-closure}) we damp relative to the top eigenvalue,
\[
    \widehat{C} = C + \max\bigl(\delta\lambda_{\max}(C),\, \varepsilon\bigr) I,
\]
which caps the condition number of $\widehat{C}$ at $(1+\delta)/\delta \approx 1/\delta$.
We apply the same relative damping to the diagonal preconditioners $P$ and $Q$.
Since the normalization in line~\ref{ln:pqnorm} sets their largest entry to
one, the damped matrices are simply $P+\delta I$ and $Q+\delta I$.

\subsection{Spectral-Norm Estimation}
\label{app:smax}
\Cref{alg:ours} requires repeated estimates of spectral norms.  
We estimate each $\norm{M}_2$ by warm-started power iteration on the smaller of the two Gram matrices $MM^\top$ and $M^\top M$ (\Cref{alg:smax}), caching the leading vector across optimizer steps so it starts near the top singular vector. A cold or stale start
can underestimate $\norm{M}_2$, and since the optimizer divides by this estimate, that would risk destabilizing training.  We therefore return the larger of the power-iteration value and the maximum row norm of $M$, which does not depend on the cached start vector. We use $K_{\mathrm{pow}} = 8$ iterations.

\begin{algorithm}[ht]
\caption{Guarded spectral-norm estimate of $\norm{M}_2$}
\label{alg:smax}
\begin{algorithmic}[1]
\Require $M\in\mathbb{R}^{m\times n}$, $m\le n$ (else $M^\top$); cached $v\in\mathbb{R}^m$; iterations $K_{\mathrm{pow}}$
\State $l \gets \max_i\norm{M_{i,:}}_2$
\If{$v$ is invalid} \Comment{absent, zero, or nonfinite}
  \State $v\gets M\mathbf{1}$
\EndIf
\For{$k=1,\dots,K_{\mathrm{pow}}$} \Comment{iterate the smaller Gram $MM^\top$}
  \State $w\gets M^\top v$;\quad $v\gets Mw/\norm{Mw}_2$
\EndFor
\State \Return $\max\big\{\norm{M^\top v}_2,\ l\big\}$
\end{algorithmic}
\end{algorithm}

\subsection{Matrix Sign and Inverse Square Root via Gram Newton--Schulz}
\label{app:gramns}

Beyond the spectral norm, \Cref{alg:ours} needs the matrix sign $\msign(\cdot)$
and the inverse square root of the $r \times r$ curvature matrix.  Both reduce to
an inverse square root of a small positive semidefinite matrix: the sign is
$\msign(X)=(XX^\top)^{-1/2}X$ (for a full row rank matrix $X$; for a tall $X$, apply to $X^\top$ and transpose the result), and the curvature term is $C^{-1/2}$ directly.  We
compute that inverse square root by the Gram Newton--Schulz iteration
\cite{gram_ns} (\Cref{alg:gramns}), which uses only $r\times r$ matrix
multiplications and needs no eigendecomposition or SVD.

\paragraph{Quintic Newton--Schulz.}
Each iteration applies an odd degree-five polynomial
\begin{equation}\label{e-gram-ns-poly}
    p_t(s)
    =
    a_t s + b_t s^3 + c_t s^5 .
\end{equation}
The coefficients $(a_t,b_t,c_t)$ are the PolarExpress coefficients
\cite{polarexpress}. They are fit so that the composition of the $K$ maps
$p_t$ drives every singular value in $[10^{-3},1]$ toward $1$. We compute
them once and reuse them for every optimizer step.

\paragraph{Gram iteration.}
Rather than iterating on the rectangular matrix, we iterate on its $r\times r$
Gram matrix \cite{gram_ns}.  Let $R_t = Y_t Y_t^\top \in \mathbb{R}^{r\times r}$
be the Gram matrix of the rectangular iterate $Y_t$.  For $t=1,\dots,K$ the
rectangular update and its Gram matrix evolve together in closed form,
\[
\begin{aligned}
    M_t &= a_t I + b_t R_{t-1} + c_t R_{t-1}^2, \\
    Y_t &= M_t Y_{t-1}, \\
    R_t &= Y_tY_t^\top = M_t R_{t-1} M_t ,
\end{aligned}
\]
so the whole iteration runs on $R_t$ alone.  Because $M_t$ is a polynomial in
$R_{t-1}$, it shares the singular vectors of $Y_{t-1}$ and moves only its
singular values; since the coefficients \eqref{e-gram-ns-poly} drive the
normalized singular values to $1$, we have $Y_K \approx \msign(Y_0)$.  The Gram update thus
implements the same quintic Newton--Schulz map on the singular values, at the
cost of only $r\times r$ matrix multiplications.

\paragraph{Shared accumulator.}
The method also maintains an accumulator $Z_t$ collecting the product of the
maps $M_t$: starting from $Z_0=I$, each step sets $Z_t=M_t Z_{t-1}$, so by
induction
\[
    Y_t=Z_tY_0,
    \qquad
    R_t=Z_tR_0Z_t^\top .
\]
Since $R_0\succ0$, $\msign(Y_0)$ has orthonormal rows, so
$R_K=Y_KY_K^\top\approx I$.  Each $M_t$ is a polynomial in $R_{t-1}$ and
hence in $R_0$, so $Z_K$ is symmetric and commutes with $R_0$; then
$R_K=Z_K^2R_0\approx I$ gives $Z_K\approx R_0^{-1/2}$.
\Cref{alg:gramns} initializes $R_0=S/\gamma$, so the returned accumulator
$Z=Z_K$ satisfies
\[
    Z\approx(S/\gamma)^{-1/2}=\sqrt{\gamma}\,S^{-1/2}.
\]
This gives both quantities \Cref{alg:ours} needs:

\begin{itemize}[leftmargin=1.5em,itemsep=2pt,topsep=2pt]
    \item \textbf{Matrix sign.}
    Call \Cref{alg:gramns} with $S=XX^\top$ and $\gamma=\norm{X}_F^2$, giving
    \[
        \msign(X)=(XX^\top)^{-1/2}X\approx ZX/\sqrt{\gamma}.
    \]
    \item \textbf{Inverse square root.}
    Call \Cref{alg:gramns} with the damped curvature $S=\widehat{C}$
    (\Cref{app:damping}) and $\gamma=\Tr(\widehat{C})$, giving
    \[
        \widehat{C}^{-1/2}\approx Z/\sqrt{\gamma}.
    \]
\end{itemize}

\paragraph{Stability and precision.}
We run the Gram Newton--Schulz iteration in fp32, which is inexpensive because
every matrix is only $r\times r$.  We use $K=8$ iterations.

\begin{algorithm}[ht]
\caption{Gram Newton--Schulz iteration}
\label{alg:gramns}
\begin{algorithmic}[1]
\Require positive definite $S \in \mathbb{R}^{r \times r}$; $\gamma \ge \lambda_{\max}(S)$; coefficients $\{(a_t,b_t,c_t)\}_{t=1}^{K}$
\Statex\vspace{-0.5\baselineskip}
\State $R\gets S/\gamma$;\quad $Z\gets I$ \Comment{normalize so $\lambda_{\max}(R)\le 1$}
\For{$t=1,\dots,K$}
  \State $M\gets a_t I+b_t R+c_t R^2$
  \State $Z\gets M Z$;\quad $R\gets M R M$
\EndFor
\State \Return $Z\approx (S/\gamma)^{-1/2}$
\end{algorithmic}
\end{algorithm}

\section{FLOP Overhead per Optimizer Step}
\label{app:flops}
To measure the theoretical overhead of each optimizer step, we count
leading-order matrix-multiplication FLOPs, with one multiply-add counted as two
FLOPs.  The count below is for one
adapted linear layer with frozen base weight $W$ and rank-$r$ LoRA factors
$A,B$,
\[
  W\in\mathbb{R}^{d_{\mathrm{out}}\times d_{\mathrm{in}}},\quad
  B\in\mathbb{R}^{d_{\mathrm{out}}\times r},\quad
  A\in\mathbb{R}^{r\times d_{\mathrm{in}}},
\]
with $r\ll d_{\mathrm{in}},d_{\mathrm{out}}$.  Throughout, $K$ denotes the number
of Newton--Schulz iterations, and $n_{\mathrm{tok}}$ denotes the number of tokens
processed per optimization step.

\paragraph{Forward and backward.}
For a linear
layer $Y=XW^\top$, the forward pass costs
$2d_{\mathrm{in}}d_{\mathrm{out}}n_{\mathrm{tok}}$
FLOPs.
In full finetuning, given the backward gradient $G_Y=\nabla_Y \mathcal{L}$,
the backward pass computes both the input and the weight gradient,
\[
\nabla_X \mathcal{L} = G_Y W,
\qquad
\nabla_W \mathcal{L} = G_Y^\top X,
\]
and costs twice the forward. LoRA freezes the base weight $W$ and does not form
$\nabla_W \mathcal{L}$, so its backward through $W$ is a single forward-equivalent matmul:
\begin{equation}
  C_{\mathrm{fwd}}^{\mathrm{base}}
  =
  C_{\mathrm{bwd}}^{\mathrm{base}}
  =
  2 d_{\mathrm{in}}d_{\mathrm{out}} n_{\mathrm{tok}} .
  \label{eq:flop-fb}
\end{equation}
The frozen base weights thus contribute $4Nn_{\mathrm{tok}}$ FLOPs per step, where $N$ is the total number of frozen adapted-layer parameters, versus the usual $6Nn_{\mathrm{tok}}$ for full finetuning~\cite{kaplan2020scaling}. The LoRA factors add $r(d_{\mathrm{in}}+d_{\mathrm{out}})$ trainable parameters per layer and FLOPs of order $n_{\mathrm{tok}}r(d_{\mathrm{in}}+d_{\mathrm{out}})$, which we omit from the count, making the overhead fraction below an overestimate.

\paragraph{Optimizer.}
The \methodname{} update (\Cref{alg:ours}) runs once per step and does not
depend on $n_{\mathrm{tok}}$.
We count the operations that scale with
$d_{\mathrm{in}}$ or $d_{\mathrm{out}}$: forming
\[
    C_B=B^\top P B,
    \qquad
    C_A=A Q A^\top,
\]
the preconditioned sign inputs, 
the small-side Gram matrices for the two
matrix signs, applying the accumulated Gram maps to recover the signs,
post-preconditioning the signs, and computing the preconditioner updates
\[
    \diag(G_A^\top C_B^{-1}G_A),
    \qquad
    \diag(G_B C_A^{-1}G_B^\top).
\]
With diagonal $P, Q$, these operations contribute
$12r^2(d_{\mathrm{in}}+d_{\mathrm{out}})$
leading FLOPs.

The Newton--Schulz iterations themselves run in the small $r\times r$ Gram
space.  One Gram Newton--Schulz iteration costs $8r^3$ FLOPs in the notation of
\Cref{app:gramns}: $2r^3$ to form $R^2$, $2r^3$ for $Z\leftarrow MZ$, and $4r^3$ for
$R\leftarrow MRM$.  Each layer uses four calls of $K$ iterations each (two matrix
signs, two curvature inverse square roots), so the total small-matrix work is $32Kr^3$.
Ignoring lower-order elementwise operations, diagonal scalings, and spectral-norm
power iterations, the leading optimizer cost is
\begin{equation}
  C_{\mathrm{opt}}
  =
  12r^2(d_{\mathrm{in}}+d_{\mathrm{out}})
  +
  32Kr^3 .
  \label{eq:flop-opt}
\end{equation}

\paragraph{Overhead fraction.} Dividing \eqref{eq:flop-opt} by \eqref{eq:flop-fb} gives the
per-layer optimizer FLOP overhead 
\begin{equation}
  \frac{C_{\mathrm{opt}}}{C_{\mathrm{fwd}}^{\mathrm{base}}
  +C_{\mathrm{bwd}}^{\mathrm{base}}}
  =
  \frac{
    12r^2(d_{\mathrm{in}}+d_{\mathrm{out}})
    +
    32Kr^3
  }{
    4d_{\mathrm{in}}d_{\mathrm{out}}n_{\mathrm{tok}}
  }
  =
  \frac{3r^2}{n_{\mathrm{tok}}}
  \left(
    \frac1{d_{\mathrm{in}}}
    +
    \frac1{d_{\mathrm{out}}}
  \right)
  +
  \frac{
    8Kr^3
  }{
    d_{\mathrm{in}}d_{\mathrm{out}}n_{\mathrm{tok}}
  } .
  \label{eq:flop-ratio}
\end{equation}
For fixed model dimensions, rank, and Newton--Schulz iteration count, the
optimizer FLOP fraction decays as $1/n_{\mathrm{tok}}$.
The first term is
governed by the smaller of the two layer dimensions. The ratio of the second term to the
first is $8Kr/(3(d_{\mathrm{in}}+d_{\mathrm{out}}))$, so the Newton--Schulz term matters
most on the square attention matrices and least on the wide feed-forward matrices.
The whole-model overhead averages these fractions with weights proportional to layer size,
so it is set mainly by the large feed-forward matrices, which have the smallest fraction.

\paragraph{Benefits of Gram Newton--Schulz.} The $32Kr^3$ cost of the Gram Newton--Schulz iteration is the only term that depends
on $K$. The iteration is used to compute the matrix sign and the curvature inverse square
roots on the small $r\times r$ Gram matrix rather than on the tall $d\times r$ factor
(\Cref{app:gramns}), saving a factor of about $d/(2r)$ per iteration. Iterating on the factor
instead would make the optimizer step roughly $2\times$ costlier, increasing its share of
the wall-clock time.

\section{Factor Self-Balancing}
\label{app:gauge}
The update size $\rho=\eta/(\norm{A}_2+\norm{B}_2)$~\eqref{eq:ours-rho} depends on the factorization of the merged product $BA$. A rescaling $(A,B)\mapsto(cA,c^{-1}B)$ with $c>0$ preserves $BA$, and hence the layer's output, while changing the denominator to
\[
c\norm{A}_2+c^{-1}\norm{B}_2 \;\ge\; 2\sqrt{\norm{A}_2\norm{B}_2}.
\]
Equality holds when the two rescaled norms are equal, so an imbalanced factorization of the same product takes smaller steps. LoRA-Muon raises this concern for Spectron-style update sizes \cite{loramuon}, but in our runs the factors stay near balance. Starting from $B=0$, the ratio $\norm{B}_2/\norm{A}_2$ climbs and settles near $1$ across ranks (\Cref{fig:gauge}).

\begin{figure}[H]
  \centering
  \includegraphics[width=0.52\linewidth]{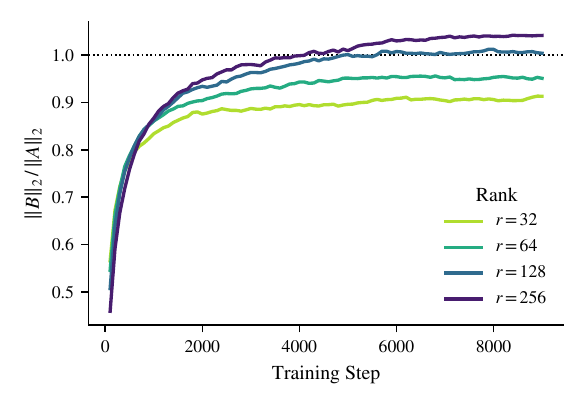}
  \caption{\textbf{Factor norm self-balancing in \methodname{}.} Spectral-norm ratio $\norm{B}_2/\norm{A}_2$ over training steps for Llama-3.2-1B finetuned on the math dataset at each rank. The ratio settles near $1$.}
  \label{fig:gauge}
\end{figure}

\section{Additional Learning Curves}
\label{app:curves}
\Cref{tab:breadth} shows the step and wall-clock speedups over Adam across model families on code and math.
\Cref{fig:curves} complements this with the evaluation loss over training steps: rows~(a)
and~(b) span the model families on code and math, and row~(c) the rank sweep. Each panel
shows Adam and \methodname{} at their optimal learning rates.

\begin{figure}[H]
  \centering
  \includegraphics[width=0.95\linewidth]{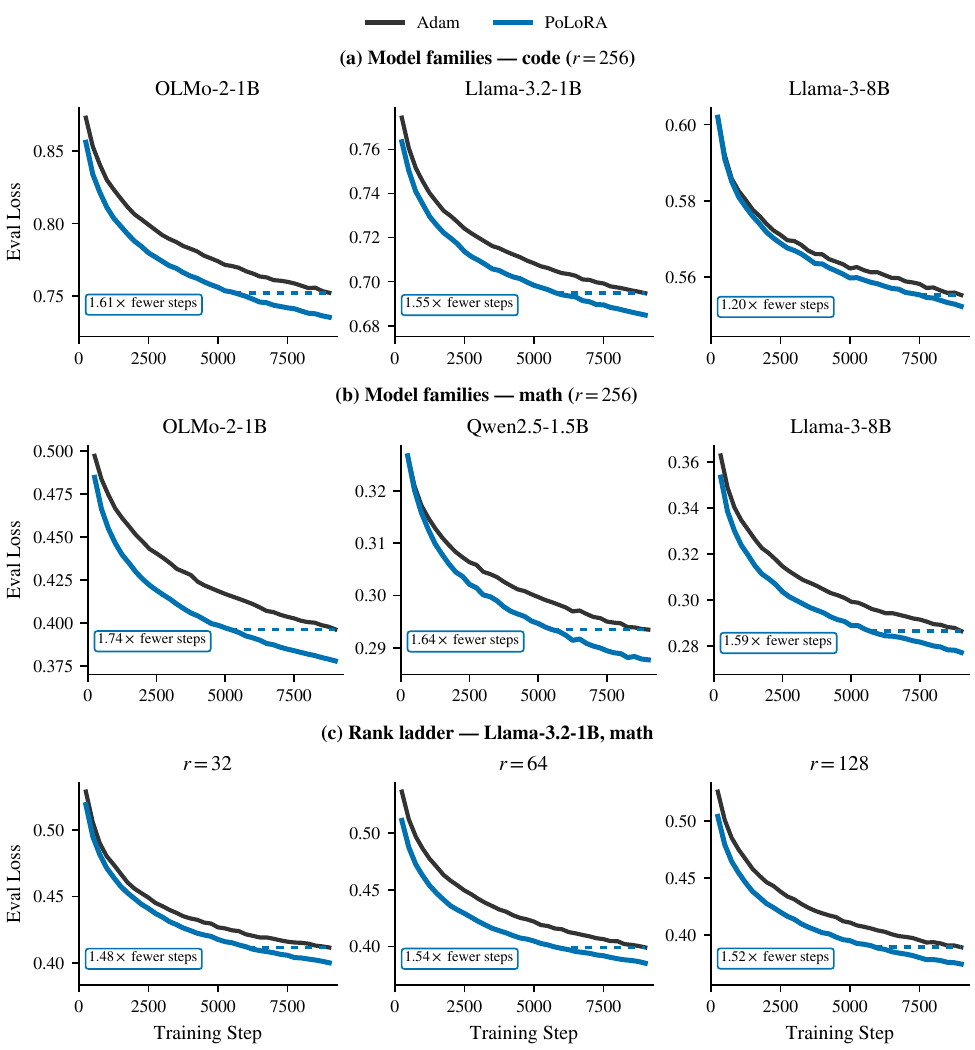}
  \caption{\textbf{\methodname{} speedup over Adam across settings.} Evaluation loss over training steps, Adam vs \methodname{} at each setting's best learning rate. \textbf{(a)}~model families finetuned on code and \textbf{(b)}~on math, at rank $r = 256$; \textbf{(c)}~the rank sweep on Llama-3.2-1B (math). The dashed span marks the steps \methodname{} saves in reaching Adam's final loss. Settings already drawn as curves in the main text (\Cref{fig:hero,fig:ablation,fig:ood}) are not repeated.}
  \label{fig:curves}
\end{figure}

\stopcontents[app]
\end{document}